\date{\today}

\documentclass{article}

\usepackage{microtype}
\usepackage{graphicx}
\usepackage{subcaption}
\usepackage{booktabs} % for professional tables
\usepackage{mathrsfs}

\usepackage{hyperref}

\usepackage[noend]{algpseudocode}

\usepackage[utf8]{inputenc} % allow utf-8 input
\usepackage[T1]{fontenc}    % use 8-bit T1 fonts
\usepackage{url}            % simple URL typesetting
\usepackage{amsfonts}       % blackboard math symbols
\usepackage{amsmath}       % blackboard math symbols
\usepackage{nicefrac}       % compact symbols for 1/2, etc.
\usepackage{microtype}      % microtypography
\usepackage{amsthm}
\usepackage[makeroom]{cancel}
\newtheorem{theorem}{Theorem}
\newtheorem{lemma}[theorem]{Lemma}

\usepackage{tikz}

\usepackage[accepted]{icml2019}

\newcommand{\ourtitle}{SATNet: Bridging deep learning and logical reasoning using a differentiable satisfiability solver}

\icmltitlerunning{\ourtitle}

\title{\ourtitle}

\DeclareMathOperator*{\minimize}{minimize}
\DeclareMathOperator*{\maximize}{maximize}
\DeclareMathOperator{\tr}{tr}
\DeclareMathOperator{\diag}{diag}
\DeclareMathOperator{\vect}{vec}

\def\subjectto{\mbox{subject to}}

\newcommand{\norm}[1] {\|#1\|}
\newcommand{\mdot}[2]{\langle#1,#2\rangle}
\newcommand{\dd}{\mathsf{d}}

\newcommand{\inndex}{\iota}
\newcommand{\outdex}{o}
\newcommand{\truthvar}{\top}

\newcommand{\inset}{\mathcal{I}}
\newcommand{\outset}{\mathcal{O}}

\usepackage{color}

\newcommand{\alignedintertext}[1]{%
  \noalign{%
    \vskip\belowdisplayshortskip
    \vtop{\hsize=\linewidth#1\par
    \expandafter}%
    \expandafter\prevdepth\the\prevdepth
    \vskip\belowdisplayshortskip
  }%
}

\begin{document}
	
	\twocolumn[
	\icmltitle{\ourtitle}
	
	
	
	\icmlsetsymbol{equal}{*}
	
	\begin{icmlauthorlist}
		\icmlauthor{Po-Wei Wang}{cmucs}
		\icmlauthor{Priya L. Donti}{cmucs,cmuepp}
		\icmlauthor{Bryan Wilder}{usc}
		\icmlauthor{Zico Kolter}{cmucs,bosch}
	\end{icmlauthorlist}
	
	\icmlaffiliation{cmucs}{School of Computer Science, Carnegie Mellon University, Pittsburgh, Pennsylvania, USA}
	\icmlaffiliation{cmuepp}{Department of Engineering \& Public Policy, Carnegie Mellon University, Pittsburgh, Pennsylvania, USA}
	\icmlaffiliation{usc}{Department of Computer Science, University of Southern California, Los Angeles, California, USA}
	\icmlaffiliation{bosch}{Bosch Center for Artificial Intelligence, Pittsburgh, Pennsylvania, USA}
	
	\icmlcorrespondingauthor{Po-Wei Wang}{poweiw@cs.cmu.edu}
	\icmlcorrespondingauthor{Priya Donti}{pdonti@cmu.edu}
	\icmlcorrespondingauthor{Bryan Wilder}{bwilder@usc.edu}
	\icmlcorrespondingauthor{Zico Kolter}{zkolter@cs.cmu.edu}
	
	\icmlkeywords{Machine Learning, ICML, MAXSAT, Deep Learning, Convex Optimization}
	
	\vskip 0.3in
	]
	
	
	
	\printAffiliationsAndNotice{}  % leave blank if no need to mention equal contribution
	
	\begin{abstract}
		Integrating logical reasoning within deep learning architectures has been a major goal of modern AI systems.  In this paper, we propose a new direction toward this goal by introducing a differentiable (smoothed) maximum satisfiability (MAXSAT) solver that can be integrated into the loop of larger deep learning systems. Our (approximate) solver is based upon a fast coordinate descent approach to solving the semidefinite program (SDP) associated with the MAXSAT problem. We show how to analytically differentiate through the solution to this SDP and efficiently solve the associated backward pass.  We demonstrate that by integrating this solver into end-to-end learning systems, we can learn the logical structure of challenging problems in a minimally supervised fashion.  In particular, we show that we can learn the parity function using single-bit supervision (a traditionally hard task for deep networks) and learn how to play $9 \times 9$ Sudoku solely from examples. We also solve a ``visual Sudoku'' problem that maps images of Sudoku puzzles to their associated logical solutions by combining our MAXSAT solver with a traditional convolutional architecture. Our approach thus shows promise in integrating logical structures within deep learning.
	\end{abstract}
	
	\section{Introduction}
	
	Although modern deep learning has produced groundbreaking improvements in a variety of domains, state-of-the-art methods still struggle to capture ``hard'' and ``global'' constraints arising from discrete logical relationships.  Motivated by this deficiency, there has been a great deal of recent interest in integrating logical or symbolic reasoning into neural network architectures \cite{palm2017recurrent, yang2017differentiable, cingillioglu2018deeplogic, evans2018learning}.  However, with few exceptions, previous work primarily focuses on integrating \emph{preexisting} relationships into a larger differentiable system via tunable continuous parameters, not on \emph{discovering} the discrete relationships that produce a set of observations in a truly end-to-end fashion.
    As an illustrative example, consider the popular logic-based puzzle game Sudoku, in which a player must fill in a $9 \times 9$ partially-filled grid of numbers to satisfy specific constraints. 
    If the rules of Sudoku (i.e.~the relationships between problem variables) are not given, then it may be desirable to jointly learn the rules of the game \emph{and} learn how to solve Sudoku puzzles in an end-to-end manner.

	We consider the problem of learning logical structure specifically as expressed by satisfiability problems -- concretely, problems that are well-modeled as instances of SAT or MAXSAT (the optimization analogue of SAT).  
	This is a rich class of domains encompassing much of symbolic AI, which has traditionally been difficult to incorporate into neural network architectures since neural networks rely on continuous and differentiable parameterizations.  
	Our key contribution is to develop and derive a differentiable smoothed MAXSAT solver that can be embedded within more complex deep architectures, and show that this solver enables effective end-to-end learning of
	% complex 
	logical relationships from examples (without hard-coding of these relationships).
	% \reword{without any hand-engineered knowledge about these logical relationships}.  
	More specifically, we build upon recent work in fast block coordinate descent methods for solving SDPs \cite{wang2017mixing} to build a differentiable solver for the smoothed SDP relaxation of MAXSAT.
	We provide an efficient mechanism to differentiate through the optimal solution of this SDP by using a similar block coordinate descent solver as used in the forward pass.  
	Our module is amenable to GPU acceleration, greatly improving training scalability.  
	% We also highlight the value and necessity of \emph{low rank} regularization methods within this framework, showing that they lead to better and much faster solutions.
	
	Using this framework, we are able to solve several problems that, despite their simplicity, prove essentially impossible for traditional deep learning methods and existing logical learning methods to reliably learn without any prior knowledge.  
	In particular, we show that we can learn the parity function, known to be challenging for deep classifiers \cite{shalev2017failures}, with only single bit supervision.  
	We also show that we can learn to play $9\times 9$ Sudoku, a problem that is challenging for modern neural network architectures \cite{palm2017recurrent}. 
	We demonstrate that our module quickly recovers the constraints that describe a feasible Sudoku solution, learning to correctly solve 98.3\% of puzzles at test time \emph{without any hand-coded knowledge of the problem structure}.
	%in less than 80 minutes of training time (and 90\% of puzzles in just 6 minutes).  
	Finally, we show that we can embed this differentiable solver into larger architectures, solving a ``visual Sudoku'' problem where the input is an image of a Sudoku puzzle rather than a binary representation. %the binary representation itself.  
	We show that, in a fully end-to-end setting, our method is able to integrate classical convolutional networks (for digit recognition) with the differentiable MAXSAT solver (to learn the logical portion).  
	Taken together, this presents a substantial advance toward a major goal of modern AI: integrating logical reasoning into deep learning architectures.

	\section{Related work}
	
	Recently, the deep learning community has given increasing attention to the concept of embedding complex, ``non-traditional'' layers within deep networks in order to train systems end-to-end.
	Major examples have included logical reasoning modules and optimization layers. 
	Our work combines research in these two areas by exploiting optimization-based relaxations of logical reasoning structures, namely an SDP relaxation of MAXSAT.
	We explore each of these relevant areas of research in more detail below.
	
	\paragraph{Logical reasoning in deep networks.} Our work is closely related to recent interest in integrating logical reasoning into deep learning architectures \cite{garcez2015neural}. 
	Most previous systems have focused on creating differentiable modules from an existing set of \emph{known relationships}, so that a deep network can learn the parameters of these relationships \cite{dai2018tunneling, manhaeve2018deepproblog, sourek2018lifted,xu2018semantic,hu2016harnessing,yang2017differentiable, selsam2018learning}.
	For example, \citet{palm2017recurrent} introduce a network that carries out relational reasoning using hand-coded information about which variables are allowed to interact, and test this network on $9 \times 9$ Sudoku.
	Similarly, \citet{evans2018learning} integrate inductive logic programming into neural networks by constructing differentiable SAT-based representations for specific ``rule templates.''
	While these networks are seeded with prior information about the relationships between variables, our approach learns these relationships \emph{and} their associated parameters end-to-end.
	While other recent work has also tried to jointly learn rules and parameters, the problem classes captured by these architectures have been limited.
	For instance, \citet{cingillioglu2018deeplogic} train a neural network to apply a specific class of logic programs, namely the binary classification problem of whether a given set of propositions entails a specific conclusion.  
	While this approach does not rely on prior hand-coded structure, our method applies to a broader class of domains, encompassing any problem reducible to MAXSAT. 
	
	\vspace{-5pt}
	\paragraph{Differentiable optimization layers.} Our work also fits within a line of research leveraging optimization as a layer in neural networks. For instance, previous work has introduced differentiable modules for quadratic programs \cite{amos2017optnet,donti2017task}, submodular optimization problems \cite{djolonga2017differentiable,tschiatschek2018differentiable,wilder2018melding}, and equilibrium computation in zero-sum games \cite{ling2018what}. To our knowledge, ours is the first work to use differentiable SDP relaxations to capture relationships between discrete variables.
	
	\vspace{-5pt}
	\paragraph{MAXSAT SDP relaxations.} We build on a long line of research exploring SDP relaxations as a tool for solving MAXSAT and related problems. 
	Classical work shows that such relaxations produce strong approximation guarantees for MAXCUT and MAX-2SAT \cite{goemans1995improved}, and are empirically tighter than standard linear programming relaxations \cite{gomes2006power}. 
	More recent work, e.g. \citet{wang2017mixing, wang2018low}, has developed low-rank SDP solvers for general MAXSAT problems.
	We extend the work of \citet{wang2017mixing} to create a differentiable optimization-based MAXSAT solver that can be employed in the loop of deep learning.
	
	\begin{figure*}[ht]
		\centering
		\includegraphics[width=\textwidth]{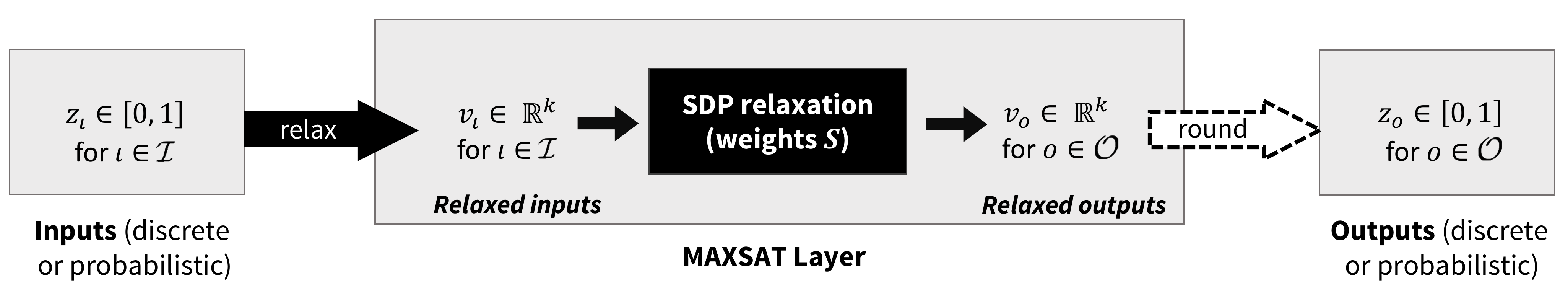}
		\caption{The forward pass of our MAXSAT layer. The layer takes as input the discrete or probabilistic assignments of known MAXSAT variables, and outputs guesses for the assignments of unknown variables via a MAXSAT SDP relaxation with weights $S$.}
		\label{fig:forward-pass}
	\end{figure*}
	
	\section{A differentiable satisfiability solver}
	The MAXSAT problem is the optimization analogue of the well-known satisfiability (SAT) problem, in which the goal is to \emph{maximize} the number of clauses satisfied.
	We present a differentiable, smoothed approximate MAXSAT solver that can be integrated into modern deep network architectures. 
	This solver uses a fast coordinate descent approach to solving an SDP relaxation of MAXSAT.
	We describe our MAXSAT SDP relaxation as well as the forward pass of our MAXSAT deep network layer (which employs this relaxation).
	We then show how to analytically differentiate through the MAXSAT SDP and efficiently solve the associated backward pass.
	
	\subsection{Solving an SDP formulation of satisfiability}
	\label{sec:sdp}
	
	Consider a MAXSAT instance with $n$ variables and $m$ clauses.
	Let $\tilde{v} \in \{-1, 1\}^{n}$ denote binary assignments of the problem variables, where $\tilde{v}_i$ is the truth value of variable $i \in \{1, \ldots, n\}$, and define $\tilde{s}_i \in \{-1, 0, 1\}^m$ for $i \in \{1, \ldots, n\}$, where $\tilde{s}_{ij}$ denotes the sign of $\tilde{v}_i$ in clause $j \in \{1, \ldots, m \}$.
	We then write the MAXSAT problem as
	\begin{equation}
		\label{eq:maxsat}
		\maximize_{\tilde{v} \in \{-1, 1\}^n} \sum_{j=1}^m \bigvee_{i=1}^n \mathbf{1} \{\tilde{s}_{ij} 
		\tilde{v}_i > 0 \}.
	\end{equation}
	
	As derived in \citet{goemans1995improved, wang2018low}, 
	%, wang2017mixing},
	to form a semidefinite relaxation of~\eqref{eq:maxsat}, we first relax the discrete variables $\tilde{v}_i$ into associated continuous variables $v_i \in \mathbb{R}^k, \; \| v_i \| = 1$ with respect to some ``truth direction'' $v_{\truthvar} \in \mathbb{R}^k, \; \|v_{\truthvar}\| = 1$.
	Specifically, we relate the continuous $v_i$ to the discrete $\tilde{v}_i$ probabilistically via $ P(\tilde{v}_i = 1) = \nicefrac{\cos^{-1}(-v_i^T v_{\truthvar})}{\pi}$ based on randomized rounding (\citet{goemans1995improved}; see Section~\ref{sec:rr}).
	We additionally define a coefficient vector $\tilde{s}_{\truthvar} = \{-1\}^m$ associated with $v_{\truthvar}$.
	Our SDP relaxation of MAXSAT is then
	\begin{equation}
		\label{eq:sdp}
		\begin{aligned}
			\minimize_{V \in \mathbb{R}^{k \times (n+1)}} \;\;&\mdot{S^T S}{V^T V},\\
			\subjectto\;\; &\norm{v_i}=1,\;\;i= \truthvar, 1,\ldots,n
		\end{aligned}
	\end{equation}
	where $V \equiv \begin{bmatrix} v_{\truthvar} & v_1 & \ldots & v_{n} \end{bmatrix} \in \mathbb{R}^{k \times (n+1)}$, and $S \equiv \begin{bmatrix} \tilde{s}_{\truthvar} & \tilde{s}_1 & \ldots & \tilde{s}_n \end{bmatrix}\diag(\nicefrac{1}{\sqrt{4|\tilde{s}_{j}|}}) \in \mathbb{R}^{m \times (n+1)}$.
	We note that this problem is a low-rank (but non-convex) formulation of MIN-UNSAT, which is equivalent to MAXSAT.
	This formulation can be rewritten as an SDP, and has been shown to recover the optimal SDP solution given $k > \sqrt{2n}$ \cite{barvinok1995problems, pataki1998rank}.
	
	Despite its non-convexity, problem~\eqref{eq:sdp} can then be solved optimally via coordinate descent for all $i=\truthvar, 1, \ldots, n$.
	In particular, the objective terms that depend on $v_i$ are given by $v_i^T \sum_{j=0}^n s_i^T s_j v_j$, where $s_i$ is the $i$th column vector of $S$.
	Minimizing this quantity over $v_i$ subject to the constraint that $\|v_i\| = 1$ yields the coordinate descent update
	\begin{equation}
		\label{eq:coord-desc-forward}
		v_i = -g_i/\norm{g_i},\;\;\text{where}\;g_i = VS^T s_i - \norm{s_i}^2 v_i.
	\end{equation}
	These updates provably converge to the globally optimal fixed point of the SDP~\eqref{eq:sdp} \cite{wang2017mixing}.
	A more detailed derivation of this update can be found in Appendix~\ref{appsec:coord-desc-fwd}.
	% We note that we can efficiently compute these coordinate descent updates in $O(nmk)$
	% % $O(2nkm)$
	% time by maintaining the result $(VS^T)$ and modifying it based on the new value of $v_i$ in each step.
	
	\begin{algorithm}[t!]
		\caption{SATNet Layer}
		\begin{algorithmic}[1]
			\Procedure{init()}{}
			\State \emph{// rank, num aux vars, initial weights, rand vectors}
			\State \textbf{init} $m, n_{\text{aux}}, S$ %, v_{\truthvar}$
% 			\State \textbf{init} $v^{\text{rand}}_i \; \forall i \in \{1, \ldots, n\}$ randomly on the unit sphere.
			\State \textbf{init} random unit vectors $v_{\truthvar}$, $v_{i}^{\text{rand}} \; \forall i \in \{1, \ldots, n\}$
			
			\State \emph{// smallest $k$ for which~\eqref{eq:sdp} recovers SDP solution}
			\State \textbf{set} $k = \sqrt{2n}+1$ \hspace{1em}
			
			\EndProcedure
			\State
			
			\Procedure{forward}{$Z_{\inset}$} \hspace{3.5em}
			\State \textbf{compute} $V_{\inset}$ \textbf{from} $Z_{\inset}$ via~\eqref{eq:input-relax}
			\State \textbf{compute} $V_{\outset}$ \textbf{from} $V_{\inset}$ via coord.\ descent (Alg~\ref{alg:forward-pass-cd})
			\State \textbf{compute} $Z_{\outset}$ \textbf{from} $V_{\outset}$ via~\eqref{eq:rr-prob}
			\State \textbf{return} $Z_{\outset}$
			\EndProcedure
			\State
			
			\Procedure{backward}{$\nicefrac{\partial \ell}{\partial Z_{\outset}}$}
			\State \textbf{compute} $\nicefrac{\partial \ell}{\partial V_{\outset}}$ via~\eqref{eq:dl-dvout}
			\State \textbf{compute} $U$ \textbf{from} $\nicefrac{\partial \ell}{\partial V_{\outset}}$  via coord.\ descent (Alg~\ref{alg:backward-pass-cd})
			\State \textbf{compute} $\nicefrac{\partial \ell}{\partial Z_{\inset}}$, $\nicefrac{\partial \ell}{\partial S}$ \textbf{from} $U$ via~\eqref{eq:dl-dzin},~\eqref{eq:dl-ds-main}
			\State \textbf{return} $\nicefrac{\partial \ell}{\partial Z_{\inset}}$
			\EndProcedure
		\end{algorithmic}
		\label{alg:satnet}
	\end{algorithm}
	
	\subsection{SATNet: Satisfiability solving as a layer}
	
	Using our MAXSAT SDP relaxation and associated coordinate descent updates, we create a deep network layer for satisfiability solving (SATNet).
	Define $\inset \subset \{1, \ldots, n\}$ to be the indices of MAXSAT variables with known assignments, and let $\outset \equiv \{1, \ldots, n\} \setminus \inset$ correspond to the indices of variables with unknown assignments.
	Our layer admits probabilistic or binary inputs $z_{\inndex} \in [0,1], \inndex \in \inset$, and then outputs the assignments of unknown variables $z_{\outdex} \in [0,1], \outdex \in \outset$ which are similarly probabilistic or (optionally, at test time) binary.
% 	\todo{(For example, in the Sudoku setting, our layer would input a bit representation of the filled-in grid cells, and output the assignments of empty grid cells.)}
	We let $Z_{\inset} \in [0,1]^{|\inset|}$ and $Z_{\outset} \in [0,1]^{|\outset|}$ refer to all input and output assignments, respectively.
	
	The outputs 
	$Z_{\outset}$  %$z_{\outdex}, \outdex \in \outset$
	are generated from inputs 
	$Z_{\inset}$  %$z_{\inndex}, \inndex \in \inset$ % 
	via the SDP~\eqref{eq:sdp}, and the weights of our layer correspond to the SDP's low-rank coefficient matrix $S$.
	% ; we note that the user can control the complexity of the MAXSAT formulation to be learned by modulating the number of rows $m$ in $S$.
	This forward pass procedure is pictured in Figure~\ref{fig:forward-pass}.
	We describe the steps of layer initialization and the forward pass in Algorithm~\ref{alg:satnet}, and in more detail below.
	
	\subsubsection{Layer initialization}
	\label{sec:aux}
	
	When initializing SATNet, the user must specify a maximum number of clauses $m$ that this layer can represent.
	It is often desirable to set $m$ to be low; in particular, \emph{low-rank structure} can prevent overfitting and thus improve generalization. 
	
	% Given the layer's low-rank structure, a user may wish to augment the layer with ``auxiliary variables.''
	Given this low-rank structure, a user may wish to somewhat increase the layer's representational ability via auxiliary variables.
% 	that are not connected to the layer's inputs or outputs.
	% Auxiliary variables are not connected to layer inputs or outputs, but instead serve to increase the layer's representational capacity.
	%without increasing its rank $m$.
	The high-level intuition here follows from 
	% analysis of 
	the conjunctive normal form (CNF) representation of boolean satisfaction problems; adding additional variables to a problem can dramatically reduce the number of CNF clauses needed to describe that problem, as these variables play a role akin to register memory that is useful for inference.
	
	Finally, we set $k = \sqrt{2n} + 1$, where here $n$ captures the number of actual problem variables in addition to auxiliary variables. 
	This is the minimum value of $k$ required for our MAXSAT relaxation~\eqref{eq:sdp} to recover the optimal solution of its associated SDP \cite{barvinok1995problems, pataki1998rank}.
	
	\subsubsection{Step 1: Relaxing layer inputs}
	\label{sec:input-relax}
	Our layer first relaxes its inputs $Z_{\inset}$ into continuous vectors for use in the SDP formulation~\eqref{eq:sdp}.
	That is, we relax each layer input $z_\inndex, \inndex \in \inset$ to an associated 
	random unit vector $v_\inndex \in \mathbb{R}^k$ so that
        \begin{equation}
                v_{\inndex}^T v_{\truthvar} = -\cos(\pi z_{\inndex}).
                \label{eq:fwd-constr}
        \end{equation}
        (This equation is derived from the probabilistic relationship described in Section~\ref{sec:sdp} between discrete variables and their continuous relaxations.)
        % To find a solution satisfying~\eqref{eq:fwd-constr}, we can first without loss of generality set the ``truth direction'' $v_{\truthvar} \in \mathbb{R}^k$ to the basis $e_1$ (since the solution is rotation-invariant).
        Constraint~\eqref{eq:fwd-constr} can be satisfied by 
        % taking
	\begin{equation}
		\label{eq:input-relax}
                v_{\inndex} = -\cos(\pi z_{\inndex})v_{\truthvar} + \sin(\pi z_{\inndex})(I_k-v_{\truthvar}v_{\truthvar}^T)v^{\text{rand}}_{\inndex},
	\end{equation}
        where $v^{\text{rand}}_{\inndex}$ is a random unit vector.
	For simplicity, we use the notation $V_{\inset} \in \mathbb{R}^{k \times |\inset|}$ (i.e.~the $\inset$-indexed column subset of $V$) to collectively refer to all relaxed layer inputs derived via Equation~\eqref{eq:input-relax}.
	
	\subsubsection{Step 2: Generating continuous relaxations of outputs via SDP}
	\label{sec:sdp-fwd}
	
	Given the continuous input relaxations 
	$V_{\inset}$,
	% of the inputs $v_{\inndex}, \, \inndex \in \inset$, 
	our layer employs the coordinate descent updates~\eqref{eq:coord-desc-forward} to compute values for continuous output relaxations $v_{\outdex}, \, \outdex \in \outset$ (which we collectively refer to as $V_{\outset} \in \mathbb{R}^{k \times |\outset|}$).
	% (which we collectively refer to as $V_{\inset} \in \mathbb{R}^{k \times |\inset|}$, i.e. the $\inset$-indexed column subset of $V$)  (which we collectively refer to as $V_{\outset} \in \mathbb{R}^{k \times |\outset|}$).
	% (We use the notation $V_{\inset} \in \mathbb{R}^{k \times |\inset|}$ and $V_{\outset} \in \mathbb{R}^{k \times |\outset|}$ to refer to all relaxed inputs and outputs, respectively.)
	Notably, coordinate descent updates are \emph{only computed for output variables}, i.e. are not computed for variables whose assignments are given as input to the layer.

	% These updates provably converge to the globally optimal fixed point of the SDP~\eqref{eq:sdp} \cite{wang2017mixing}.
	% This convergence property enables efficient computation of the gradient through the SDP solution, which we will detail in Section~\ref{sec:backward}.
	
	Our coordinate descent algorithm for the forward pass is detailed in Algorithm~\ref{alg:forward-pass-cd}.
	This algorithm maintains the term $\Omega = VS^T$ needed to compute $g_{\outdex}$, and then modifies it via a rank-one update during each inner iteration.
	Accordingly, the per-iteration runtime is $O(nmk)$ (and in practice, only a small number of iterations is required for convergence).
	
	\begin{algorithm}[t]
		\begin{algorithmic}[1]
			\caption{Forward pass coordinate descent}
			\State \textbf{input} $V_{\inset}$ \hspace{3.5em} \emph{// inputs for known variables}
			\State \textbf{init} $v_{\outdex}$ with random vector $v^{\text{rand}}_{\outdex},\;\forall \outdex \in \outset$.
			\State \textbf{compute} $\Omega = VS^T$
			\While{not converged}
			\For{$\outdex \in \outset$} \hspace{1em} \emph{// for all output variables}
			\State \textbf{compute} $g_{\outdex} = \Omega s_{\outdex} - \|s_{\outdex}\|^2 v_{\outdex}$ as in~\eqref{eq:coord-desc-forward}
			\State \textbf{compute} $v_{\outdex} = -g_{\outdex}/\|g_{\outdex}\|$ as in~\eqref{eq:coord-desc-forward}
			% \State \emph{// rank 1 update}
			\State \textbf{update} $\Omega = \Omega + (v_{\outdex} - v_{\outdex}^{\text{prev}})s_{\outdex}^T$
			\EndFor
			\EndWhile
			\State \textbf{output} $V_{\outset}$ \hspace{3.5em} \emph{// final guess for output cols of $V$}
			\label{alg:forward-pass-cd}
		\end{algorithmic}
	\end{algorithm}
	
	\subsubsection{Step 3: Generating discrete or probabilistic outputs}
	\label{sec:rr}
	
	Given the relaxed outputs 
	% $v_{\outdex},\, \outdex \in \outset$ 
	$V_{\outset}$
	from coordinate descent, our layer converts these outputs to discrete or probabilistic variable assignments 
	% $z_{\outdex}$ 
	$Z_{\outset}$
	via either thresholding or randomized rounding (which we describe here). 
	
	The main idea of randomized rounding is that for every $v_{\outdex}, \outdex \in \outset$, we can take a random hyperplane $r$ from the unit sphere and assign
	\begin{equation}
		\tilde{v}_{\outdex} = \begin{cases}
			1& \text{if }\text{sign}(v_{\outdex}^Tr) = \text{sign}(v_{\truthvar}^T r)\\
			-1&\text{otherwise}
		\end{cases}, \; \outdex \in \outset,
	\end{equation}
	where $\tilde{v}_{\outdex}$ is the boolean output for $v_{\outdex}$. 
	Intuitively, this scheme sets $\tilde{v}_{\outdex}$ to ``true'' if and only if $v_{\outdex}$ and the truth vector $v_{\truthvar}$ are on the same side of the random hyperplane $r$.
	Given the correct weights $S$, this randomized rounding procedure assures an optimal expected approximation ratio for certain NP-hard problems \cite{goemans1995improved}. 
	
	During training, we do not explicitly perform randomized rounding.
	We instead note that the probability that $v_{\outdex}$ and $v_{\truthvar}$ are on the same side of any given $r$ is
	\begin{equation}
		\label{eq:rr-prob}
		P(\tilde{v}_{\outdex}) = \cos^{-1}(-v_{\outdex}^T v_{\truthvar})/\pi,
	\end{equation}
	and thus set $z_{\outdex} = P(\tilde{v}_{\outdex})$ to equal this probability. 
	
	During testing, we can either output probabilistic outputs in the same fashion, or output discrete assignments via thresholding or randomized rounding.
	If using randomized rounding, we round multiple times, and then
	set $z_{\outdex}$ to be the 
% 	pick the 
	boolean solution maximizing the MAXSAT objective in Equation~\eqref{eq:maxsat}. 
	Prior work has observed that such repeated rounding improves approximation ratios in practice, especially for MAXSAT problems \cite{wang2018low}.
	% , which can sometimes be solved directly via this scheme without tree search \cite{wang2018low}.
% 	We then set $z_{\outdex} = (\tilde{v}_{\outdex}+1)/2$.

	\subsection{Computing the backward pass}
	\label{sec:backward}
	
	We now derive backpropagation updates through our SATNet layer to enable its integration into a neural network.
	That is, given the gradients $\nicefrac{\partial \ell}{\partial Z_{\outset}}$ of the network loss $\ell$ with respect to the layer outputs, we must compute the gradients $\nicefrac{\partial \ell}{\partial  Z_{\inset}}$ with respect to layer inputs and $\nicefrac{\partial \ell}{\partial S}$ with respect to layer weights.
	As it would be inefficient in terms of time and memory to explicitly unroll the forward-pass computations and store intermediate Jacobians,
	we instead derive analytical expressions to \emph{compute the desired gradients directly}, employing an efficient coordinate descent algorithm.
	The procedure for computing these gradients is summarized in Algorithm~\ref{alg:satnet} and derived below.
	
	\subsubsection{From probabilistic outputs to their continuous relaxations}
	
	Given $\nicefrac{\partial \ell}{\partial Z_{\outset}}$ (with respect to the layer outputs), we first derive an expression for $\nicefrac{\partial \ell}{\partial V_{\outset}}$ (with respect to the output relaxations) by pushing gradients through the probability assignment mechanism described in Section~\ref{sec:rr}.
	That is, for each $\outdex \in \outset$,
	\begin{equation}
		\label{eq:dl-dvout}
		\frac{\partial \ell}{\partial v_{\outdex}} 
		= \left(  \frac{\partial \ell}{\partial z_{\outdex}} \right) \left( \frac{\partial z_{\outdex}}{\partial v_{\outdex}} \right) 
		=  \left(  \frac{\partial \ell}{\partial z_{\outdex}} \right) \frac{1}{\pi \sin(\pi z_{\outdex})} v_{\truthvar} ,
	\end{equation}
	where we obtain $\nicefrac{\partial z_{\outdex}}{\partial v_{\outdex}}$ by differentiating through Equation~\eqref{eq:rr-prob} (or, more readily, by implicitly differentiating through its rearrangement 
	$\cos(\pi z_{\outdex}) = -v_{\truthvar}^T v_{\outdex}$).

	\subsubsection{Backpropagation through the SDP}
	\label{sec:sdp-back}
	
	Given the analytical form for $\nicefrac{\partial \ell}{\partial V_{\outset}}$ (with respect to the output relaxations), we next seek to derive $\nicefrac{\partial \ell}{\partial V_{\inset}}$ (with respect to the input relaxations) and $\nicefrac{\partial \ell}{\partial S}$ (with respect to the layer weights) by pushing gradients through our SDP solution procedure (Section~\ref{sec:sdp-fwd}).
	We describe the analytical form for the resultant gradients in Theorem~\ref{thm:back-grads}.
	
	\begin{theorem}
		Define $P_{\outdex} \equiv I_k-v_{\outdex} v_{\outdex}^T$ for each $\outdex \in \outset$. Then, define $U \in \mathbb{R}^{k \times n}$, where the columns $U_{\inset} = 0$ and the columns $U_{\outset}$ are given by
		\begin{equation}
			\vect(U_{\outset})=  \left(P((C+D)\otimes I_k)P\right)^\dagger\vect\left(\frac{\partial \ell}{\partial V_{\outset}}\right),
			\label{eq:mv-solve}
		\end{equation}
	    where $P \equiv \diag(P_\outdex)$, where $C\equiv S_{\outset}^TS_{\outset}-\diag(\norm{s_{\outdex}}^2)$, and where $D \equiv \diag(\|g_{\outdex}\|)$. 
		%
% 		Further, define $\hat{U} \in \mathbb{R}^{k \times n}$, where the columns $\hat{U}_{\inset} = 0$ and the columns $\hat{U}_{\outset}$ are given by $\hat{u}_{\outdex} \equiv P_{\outdex} u_{\outdex}$ for $\outdex \in \outset$. 
		Then, the gradient of the network loss $\ell$ with respect to the relaxed layer inputs is
		\begin{equation}
			\frac{\partial \ell}{\partial V_{\inset}} = -\Big(\sum_{\outdex \in \outset} u_{\outdex} s_{\outdex}^T \Big)S_{\inset}, \label{eq:dl-dvin-main}
		\end{equation}
	where $S_{\inset}$ is the $\inset$-indexed column subset of $S$, and the gradient with respect to the layer weights is
		\begin{equation}
			\frac{\partial \ell}{\partial S} = -\Big(\sum_{\outdex \in \outset} u_{\outdex} s_{\outdex}^T \Big)^TV - (SV^T)U. \label{eq:dl-ds-main}
		\end{equation}
		\label{thm:back-grads}
	\end{theorem}
	
	We defer the derivation of Theorem~\ref{thm:back-grads} to Appendix~\ref{appsec:sdp-backprop}.
	Although this derivation is somewhat involved, the concept at a high level is quite simple: we differentiate the solution of the SDP problem (Section~\ref{sec:sdp}) with respect to the problem's parameters and input, which requires computing the (relatively large) matrix-vector solve given in Equation~\eqref{eq:mv-solve}. 
	
	To solve Equation~\eqref{eq:mv-solve}, %compute the solution for this linear system, 
	we use a coordinate descent approach that closely mirrors the coordinate descent procedure employed in the forward pass, and which has similar fast convergence properties.
	% We describe this derivation for one example parameter below, and give the full derivation for all parameters in Appendix~\ref{appsec:sdp-backprop}. 
	This procedure, described in Algorithm~\ref{alg:backward-pass-cd}, 
	% closely mirrors the coordinate descent procedure employed in the forward pass
	% \todo{, and in fact we perform both coordinate descent procedures using the same solver}.
	% Our formulation of Algorithm~\ref{alg:backward-pass-cd} 
	% and
	enables us to compute the desired gradients without needing to maintain intermediate Jacobians explicitly.
	% As such, we can use the same solver for coordinate descent in the forward and backward passes, by exploiting their similar structure.
	Mirroring the forward pass, we use rank-one updates to maintain and modify the term $\Psi = U S^T$ needed to compute $\dd g_{\outdex}$, which again enables our algorithm to run in $O(nmk)$ time.
	We defer the derivation of Algorithm~\ref{alg:backward-pass-cd} to Appendix~\ref{appsec:back-cd}.
	
	\subsubsection{From relaxed to original inputs}
	
	As a final step, we must use the gradient $\nicefrac{\partial \ell}{\partial V_{\inset}}$ (with respect to the input relaxations) to derive the gradient $\nicefrac{\partial \ell}{\partial Z_{\inset}}$ (with respect to the actual inputs) by pushing gradients through the input relaxation procedure described in Section~\ref{sec:input-relax}.
	For each $\inndex \in \inset$, we see that
	\begin{equation}
		\begin{aligned}
			\frac{\partial \ell}{\partial z_{\inndex}} &= \frac{\partial \ell}{\partial z_{\inndex}^{\star}} + \left(\frac{\partial \ell}{\partial v_{\inndex}}  \right)^T \frac{\partial v_{\inndex}}{\partial z_{\inndex}} \\
			&= \frac{\partial \ell}{\partial z_{\inndex}^{\star}} -  \left(\frac{\partial v_{\inndex}}{\partial z_{\inndex}} \right)^T \left(\sum_{\outdex \in \outset} u_{\outdex} s_{\outdex}^T \right) s_{\inndex}
		\end{aligned}
		\label{eq:dl-dzin}
	\end{equation}
	where
	\begin{equation}
			\frac{\partial v_{\inndex}}{\partial z_{\inndex}} = 
			\pi\left(\sin (\pi z_{\inndex})v_{\truthvar} + \cos (\pi z_{\inndex}) (I_k - v_{\truthvar} v_{\truthvar}^T) v_{\inndex}^{\text{rand}} \right),
	\label{eq:dv-dz}
	\end{equation}
	and where $\nicefrac{\partial \ell}{\partial z_{\inndex}^{\star}}$ captures any direct dependence of $\ell$ on $z_{\inndex}^{\star}$ (as opposed to dependence through $v_{\inndex}$).
	Here, the expression for $\nicefrac{\partial \ell}{\partial v_{\inndex}}$ comes from Equation~\eqref{eq:dl-dvin-main}, and we obtain $\nicefrac{\partial v_{\inndex}}{\partial z_{\inndex}}$ by differentiating Equation~\eqref{eq:input-relax}.
% 	Equation~\eqref{eq:dl-dzin} is the analytical form of the input gradient $\nicefrac{\partial \ell}{\partial Z_{\inset}}$ used by our layer.
	
	\begin{algorithm}[t]
		\begin{algorithmic}[1]
			\caption{Backward pass coordinate descent}
			\State \textbf{input} $\{\nicefrac{\partial \ell}{\partial v_{\outdex}}\;|\; \outdex \in \outset \}$ \hspace{0.3em} \emph{// grads w.r.t. relaxed outputs}
			%\State \textbf{compute} $q_{\outdex} = P_{\outdex}\frac{\partial \ell}{\partial v_{\outdex}}, \; \forall \outdex \in \outset$ %\hspace{1em} \emph{// projected backprop grads}
			\vspace{0.25em}
			\State \emph{// Compute 
				% $Q\big(D + C \big)^{\dagger}$ and store in the variable $U_{\outset}$
				$U_{\outset}$ from Equation~\eqref{eq:mv-solve}}
			% \State \emph{// \hspace{0.5em} coordinate descent and store in the variable $U$.}
			\State \textbf{init} $U_{\outset} = 0$ and $\Psi = (U_{\outset})S^T_{\outset} = 0$
			\While{not converged}
			\For{$\outdex \in \outset$} \hspace{1em} \emph{// for all output variables}
			\State \textbf{compute} $\dd g_{\outdex} = \Psi s_{\outdex} - \|s_{\outdex}\|^2 u_{\outdex} - \partial\ell/\partial v_{\outdex}$.
			\State \textbf{compute} $u_{\outdex} = -P_{\outdex}\dd g_{\outdex}/\|g_{\outdex}\|$.
			% \State \emph{// rank 1 update}
			\State \textbf{update} $\Psi = \Psi + (u_{\outdex} - u_{\outdex}^{\text{prev}})s_{\outdex}^T $
			\EndFor
			\EndWhile
			\State \textbf{output} $U_{\outset}$
			\label{alg:backward-pass-cd}
		\end{algorithmic}
	\end{algorithm}
	
	% \subsubsection{Computing backpropagation terms via coordinate descent}
	% \label{sec:coord-desc-back}
	
	% As mentioned in Section~\ref{sec:sdp-back}, we solve the linear system in Equation~\eqref{eq:mv-solve}
	% % $U=(C+D)^\dagger Q$ 
	% via coordinate descent.
	% This procedure, described in Algorithm~\ref{alg:backward-pass-cd}, closely mirrors the coordinate descent procedure employed in the forward pass, and in fact we perform both coordinate descent procedures using the same solver.
	% Our formulation of Algorithm~\ref{alg:backward-pass-cd} enables us to compute the desired gradients without needing to maintain intermediate Jacobians explicitly.
	% % As such, we can use the same solver for coordinate descent in the forward and backward passes, by exploiting their similar structure.
	% Mirroring the forward pass, we use rank-one updates to maintain and modify the term $\Psi = U S^T$ needed to compute $\dd g_{\outdex}$, which again enables our algorithm to run in $O(nmk)$ time.
	% We defer the derivation of Algorithm~\ref{alg:backward-pass-cd} to Appendix~\ref{appsec:back-cd}.
	
	\subsection{An efficient GPU implementation}
	The coordinate descent updates in Algorithms~\ref{alg:forward-pass-cd} and~\ref{alg:backward-pass-cd} dominate the computational costs of the forward and backward passes, respectively.
	We thus present an efficient, parallel GPU implementation of these algorithms to speed up training and inference. 
	During the inner loop of coordinate descent, our implementation parallelizes the computation of all $g_{\outdex}$ ($\dd g_{\outdex}$) terms by parallelizing the computation of $\Omega$ ($\Psi$), as well as of all rank-one updates of $\Omega$ ($\Psi$).
	This 
	% ability to parallelize 
	underscores the benefit of using a low-rank SDP formulation in our MAXSAT layer, as traditional full-rank coordinate descent cannot be efficiently parallelized.
	We find in our preliminary benchmarks that our GPU CUDA-C implementation is up to $18-30$x faster than the corresponding OpenMP implementation run on Xeon CPUs.
	% , reducing the training time of our experiments from hours to minutes.
	Source code for our implementation is available at \url{https://github.com/locuslab/SATNet}.
	% Both the forward and backward pass calculations can be simplified as a block coordinate descent method on $V$ or $\dd V$ (with varying diagonal).
	% At each inner iteration, we perform a matrix vector multiplication on $VS^T$ or $\left(\sum_{\outdex\in\outset}q_{\outdex} s_{\outdex}^T\right)$, normalize the result with the corresponding diagonal, and update the matrix with a rank-one update.
	% Inherently, the block coordinate descent method is sequetial between blocks, thus hard to parallelize on GPU. However, 
	% the matrix-vector multiplication and rank-one update can be parallelize, and the low-rank constraint can be exploited.
	% For each of the GPU CUDA kernel call,
	% we process each minibatch parallelly in grids. In each grid, the threads is grouped into warps, and we assign each warp to process a independent dot product for a coordinate to parallelize the matrix vector product. Further, the dot product is implemented with the ``warp-level shuffle intrinsic`` so that the synchronization inside the warp be minimized. To maximize the GPU occupancy, we preprocess the index of active variables into a compact array, solving coordinates in the array sequentially. We also perform random permutation at the index set at the beginning of each kernel call to randomize the updating order. In summary, the GPU CUDA-C implementation is 6 to 10 times faster than the corresponding OpenMP implementation on the Xeon CPUs, reducing the training time of our experiments from hours to minutes.
	
	\section{Experiments}
	
	We test our MAXSAT layer approach in three domains that are traditionally difficult for neural networks: learning the parity function with single-bit supervision, learning $9 \times 9$ Sudoku solely from examples, and solving a ``visual Sudoku'' problem that generates the logical Sudoku solution given an input image of a Sudoku puzzle.
	We find that in all cases, we are able to perform substantially better on these tasks than previous deep learning-based approaches.

	\subsection{Learning parity (chained XOR)}
	%random = $-\ln 0.5=0.6931$
	
	This experiment tests SATNet's ability to differentiate through many successive SAT problems by learning to compute the parity function. 
	The parity of a bit string is defined as one if there is an odd number of ones in the sequence and zero otherwise. 
	The task is to map input sequences to their parity, given a dataset of example sequence/parity pairs. Learning parity functions from such single-bit supervision is known to pose difficulties for conventional deep learning approaches \cite{shalev2017failures}. 
	However, parity is simply a logic function -- namely, a sequence of XOR operations applied successively to the input sequence.
	
	% Since parity can be represented as a sequence of XOR operations applied successively to the input sequence. 
	Hence, for a sequence of length $L$, we construct our model to contain a sequence of $L-1$ SATNet layers with tied weights (similar to a recurrent network). 
	The first layer receives the first two binary values as input, and layer $d$ receives value $d$ along with the rounded output of layer $d-1$. If each layer learns to compute the XOR function, the combined system will correctly compute parity. 
	However, this requires the model to coordinate a long series of SAT problems without any intermediate supervision. 
	
	Figure~\ref{fig:parity}  shows that our model accomplishes this task for input sequences of length $L = 20$ and $L=40$.
	For each sequence length, we generate a dataset of 10K random examples (9K training and 1K testing). 
	We train our model using cross-entropy loss and the Adam optimizer \cite{kingma2015adam} with a learning rate of $10^{-1}$. 
	We compare to an LSTM sequence classifier, which uses 100 hidden units and a learning rate of $10^{-3}$ (we tried varying the architecture and learning rate but did not observe any improvement). 
	In each case, our model quickly learns the target function, with error on the held-out set converging to zero within 20 epochs. 
	In contrast, the LSTM is unable to learn an appropriate representation, with only minor improvement over the course of 100 training epochs; across both input lengths, it achieves a testing error rate of at best 0.476 (where a random guess achieves value 0.5).   
	
	\begin{figure}
		\centering
		\includegraphics[width=0.45\textwidth]{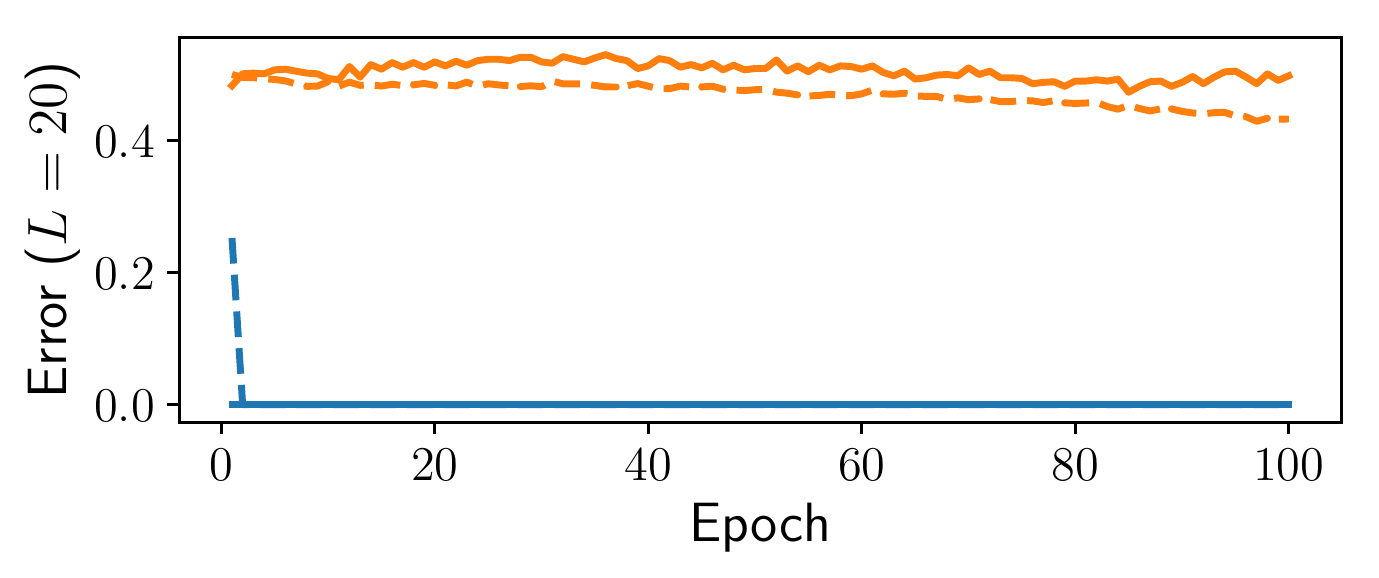} \\
		\includegraphics[width=0.45\textwidth]{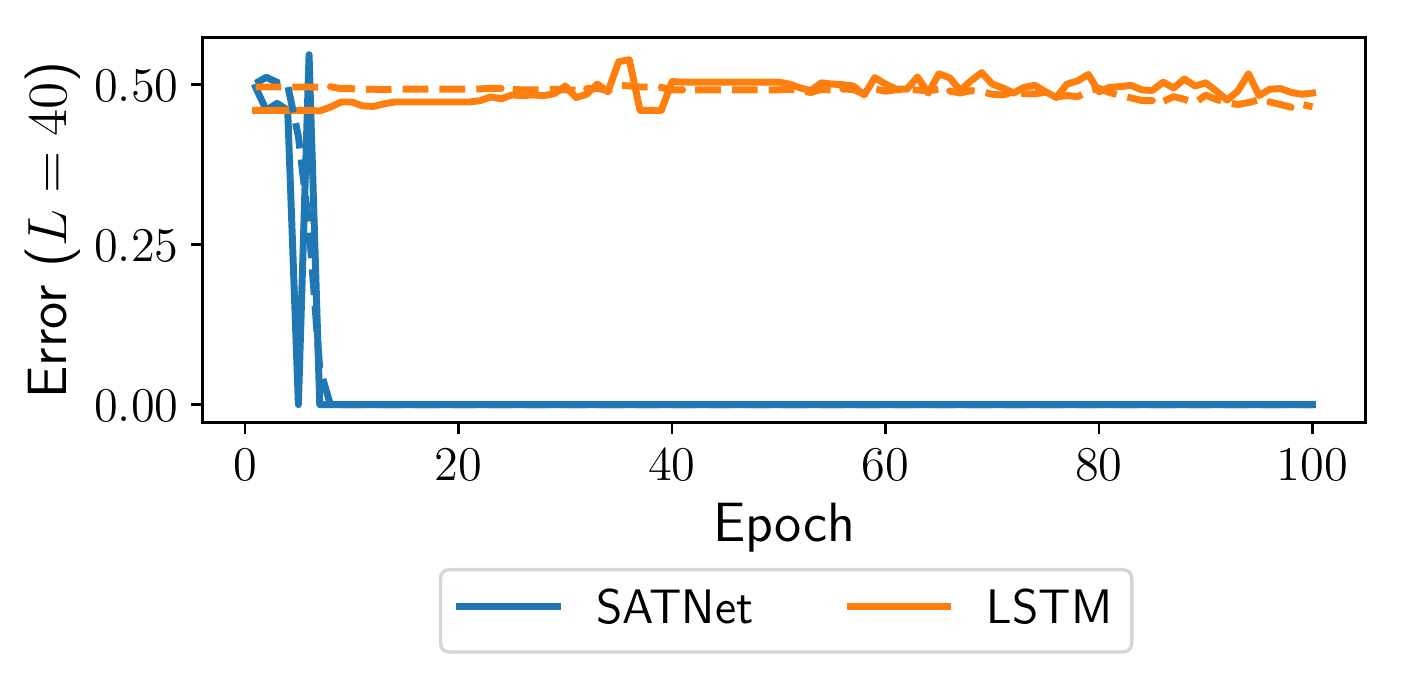}
		\caption{Error rate for the parity task with $L = 20$ (top) and $L = 40$ (bottom). Solid lines denote test values, while dashed lines represent training values.} 
		\label{fig:parity}
	\end{figure}
	
	\subsection{Sudoku (original and permuted)}
	\label{sec:sudoku}
	
	\begingroup
    \renewcommand*{\arraystretch}{1.2}
	\begin{table*}[t!]
	    \centering
	    \begin{subtable}[t]{0.33\textwidth}
        \centering
	    \begin{tabular}{ccc}
	    \toprule
	       \textbf{Model}  &  \textbf{Train} &  \textbf{Test} \\
	   \midrule
	       ConvNet & 72.6\% & 0.04\% \\
	   %\hline
	       ConvNetMask & 91.4\% & 15.1\% \\
	   %\hline
	       \textbf{SATNet (ours)}  & 99.8\% & \textbf{98.3\%}    \\
	   \bottomrule
	    \end{tabular}
	    \caption{Original Sudoku.}
	    \end{subtable}%
	  	\begin{subtable}[t]{0.33\textwidth}
        \centering
	    \begin{tabular}{ccc}
	    \toprule
	       \textbf{Model}  &  \textbf{Train} &  \textbf{Test} \\
	   \midrule
	       ConvNet & 0\% & 0\% \\
	   %\hline
	       ConvNetMask & 0.01\% & 0\% \\
	   %\hline
	    \textbf{SATNet (ours)}  & 99.7\% &  \textbf{98.3\%}   \\
	   \bottomrule
	    \end{tabular}
	    \caption{Permuted Sudoku.}
	    \end{subtable}
	    \begin{subtable}[t]{0.33\textwidth}
        \centering
	    \begin{tabular}{ccc}
	    \toprule
	       \textbf{Model}  &  \textbf{Train} &  \textbf{Test} \\
	   \midrule
	       ConvNet & 0.31\% & 0\% \\
	   %\hline
	       ConvNetMask & 89\% & 0.1\% \\
	   %\hline
	    \textbf{SATNet (ours)}  & 93.6\% &  \textbf{63.2\%} \\
	   \bottomrule
	    \end{tabular}
	    \caption{Visual Sudoku. (Note: the theoretical ``best'' test accuracy for our architecture is 74.7\%.)}
	    \end{subtable}
	    \caption{Results for $9\times 9$ Sudoku experiments with 9K train/1K test examples. We compare our SATNet model against a vanilla convolutional neural network (ConvNet) as well as one that receives a binary mask indicating which bits need to be learned (ConvNetMask).
	   }
	    \label{tab:sudoku-9}
	\end{table*}
	\endgroup
	
	In this experiment, we test SATNet's ability to infer and recover constraints simply from bit supervision (i.e. without any hard-coded specification of how bits are related).
	We demonstrate this property via 
% 	the popular logic-based number puzzle 
	Sudoku.
	In Sudoku, given a (typically) $9 \times 9$ partially-filled grid of numbers, a player must fill in the remaining empty grid cells such that each row, each column, and each of nine $3 \times 3$ subgrids contains exactly one of each number from 1 through 9.
	While this constraint satisfaction problem is computationally easy to solve once the rules of the game are specified, actually \emph{learning the rules of the game}, i.e. the hard constraints of the puzzle, has proved challenging for traditional neural network architectures.
	In particular, Sudoku problems are often solved computationally via tree search, and while tree search cannot be easily performed by neural networks, it is easily expressible 
	using
% 	in the context of 
	SAT and MAXSAT problems.
	
	We construct a SATNet model for this task that takes as input a logical (bit) representation of the initial Sudoku board along with a mask representing which bits must be learned (i.e.~all bits in empty Sudoku cells). 
	This input is vectorized, which means that our SATNet model cannot exploit the locality structure of the input Sudoku grid when learning to solve puzzles.
	Given this input, the SATNet layer then outputs a bit representation of the Sudoku board with guesses for the unknown bits.
	Our model architecture consists of a single SATNet layer with 300 auxiliary variables and low rank structure $m =600$, and we train it to minimize a digit-wise negative log likelihood objective (optimized via Adam with a $2\times 10^{-3}$ learning rate).
	
	We compare our model to a convolutional neural network baseline modeled on that of \citet{park2016can}, which interprets the bit inputs as 9 input image channels (one for each square in the board) and uses a sequence of 10 convolutional layers (each with 512 3$\times$3 filters) to output the solution. The ConvNet makes explicit use of locality in the input representation since it treats the nine cells within each square as a single image. We also compare to a version of the ConvNet which receives a binary mask indicating which bits need to be learned (ConvNetMask). The mask is input as a set of additional image channels in the same format as the board. We trained both architectures using mean squared error (MSE) loss (which gave better results than negative log likelihood for this architecture). 
	The loss was optimized using Adam (learning rate $10^{-4}$).  
	We additionally tried to train an OptNet \cite{amos2017optnet} model for comparison, but this model made little progress even after a few days of training.
	(We compare our method to OptNet on a simpler $4 \times 4$ version of the Sudoku problem in Appendix~\ref{appsec:sudoku-4}.)
	
	Our results for the traditional $9 \times 9$ Sudoku problem (over 9K training examples and 1K test examples) are shown in Table~\ref{tab:sudoku-9}. (Convergence plots for this experiment are shown in Appendix~\ref{appsec:sudoku-9}.)
	Our model is able to 
% 	perfectly 
	learn the constraints of the Sudoku problem, achieving high accuracy early in the training process (%
% 	89\% test accuracy in 22 epochs/22 minutes, and
	95.0\% test accuracy in 22 epochs/37 minutes on a GTX 1080 Ti GPU), and demonstrating 98.3\% board-wise test accuracy after 100 training epochs (172 minutes).
	On the other hand, the ConvNet baseline does poorly. It learns to correctly solve 72.6\% of puzzles in the training set but fails altogether to generalize: accuracy on the held-out set reaches at most 0.04\%. The ConvNetMask baseline, which receives a binary mask denoting which entries must be completed, performs only somewhat better, correctly solving 15.1\% of puzzles in the held-out set. 
	We note that our test accuracy is qualitatively similar to the results obtained in \citet{palm2017recurrent}, but that our network is able to learn the structure of Sudoku \emph{without explicitly encoding the relationships between variables}.
	
	To underscore that our architecture truly learns the rules of the game, as opposed to overfitting to locality or other structure in the inputs, we test our SATNet architecture on \emph{permuted} Sudoku boards, i.e. boards for which we apply a fixed permutation of the underlying bit representation (and adjust the corresponding input masks and labels accordingly).
	This removes any locality structure, and the resulting Sudoku boards do not have clear visual analogues that can be solved by humans. However, the relationships between bits are unchanged (modulo the permutation) and should therefore be discoverable by architectures that can truly learn the underlying logical structure.
	Table~\ref{tab:sudoku-9} shows results for this problem in  comparison to the convolutional neural network baselines.
	Our architecture is again able to 
% 	perfectly 
	learn the rules of the (permuted) game, demonstrating the same 98.3\% board-wise test accuracy as in the original game.
	In contrast, the convolutional neural network baselines perform even more poorly than in the original game (achieving 0\% test accuracy even with the binary mask as input), as there is little locality structure to exploit.
	Overall, these results demonstrate that SATNet can truly learn the logical relationships between discrete variables.
	
	\subsection{Visual Sudoku}
	
	\begin{figure}[t]
		\centering
		\includegraphics[width=0.3\textwidth]{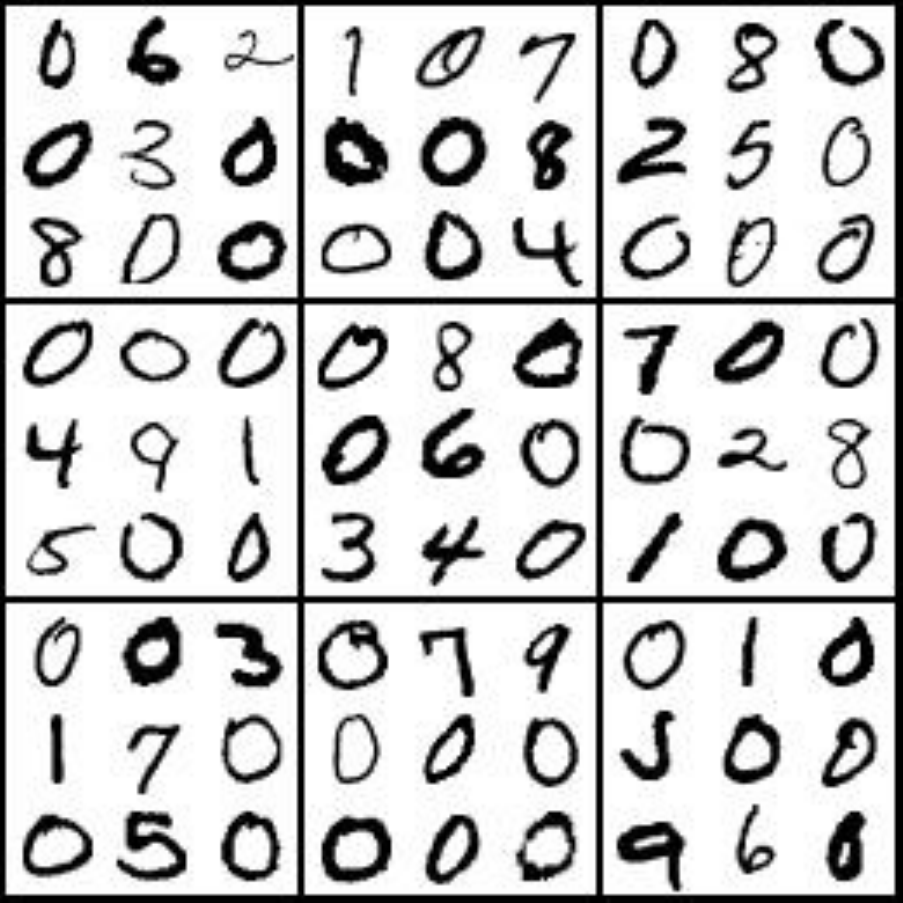}
		\caption{An example visual Sudoku image input, i.e. an image of a Sudoku board constructed with MNIST digits. Cells filled with the numbers 1-9 are fixed, and zeros represent unknowns.}
		\label{fig:mnist-sudoku-input}
	\end{figure}
	
	In this experiment, we demonstrate that SATNet can be integrated into larger deep network architectures for end-to-end training. 
	% in order to train these systems end-to-end.
	Specifically, we solve the visual Sudoku problem: that is, given an \emph{image representation} of a Sudoku board (as opposed to a one-hot encoding or other logical representation) constructed with MNIST digits, our network must output a \emph{logical solution} to the associated Sudoku problem. 
	An example input is shown in Figure~\ref{fig:mnist-sudoku-input}.
	This problem cannot traditionally be represented well by neural network architectures, as it requires the ability to combine multiple neural network layers \emph{without} hard-coding 
	% structural knowledge 
	the logical structure
	of the problem into intermediate logical layers.
	
	Our architecture for this problem uses a convolutional neural network connected to a SATNet layer. 
	Specifically, we apply a convolutional layer for digit classification (which uses the LeNet architecture \cite{lecun1998gradient}) to each cell of the Sudoku input.
	Each cell-wise probabilistic output of this convolutional layer is then fed as logical input to the SATNet layer, along with an input mask (as in Section~\ref{sec:sudoku}).
	This SATNet layer employs the same architecture and training parameters as described in the previous section.
	The whole model is trained end-to-end to minimize cross-entropy loss, and is optimized via Adam with learning rates $2\times 10^{-3}$ for the SATNet layer and $10^{-5}$ for the convolutional layer. 
	
	We compare our approach against a convolutional neural network which combines two sets of convolutional layers. First, the visual inputs are passed through the same convolutional layer as in our SATNet model, which outputs a probabilistic bit representation. Next, this representation is passed through the convolutional architecture that we compared to for the original Sudoku problem, which outputs a solution. We use the same training approach as above. 
	
	Table~\ref{tab:sudoku-9} summarizes our experimental results (over 9K training examples and 1K test examples); additional plots are shown in Appendix~\ref{appsec:sudoku-9}.
	We contextualize these results against the theoretical ``best'' testing accuracy of 74.7\%, which accounts for the Sudoku digit classification accuracy of our specific convolutional architecture;
	that is, assuming boards with 36.2 out of 81 filled cells on average (as in our test set) and 
% 	a ``perfect'' 
    an
	MNIST model with 99.2\% test accuracy \cite{lecun1998gradient}, we would expect a perfect Sudoku solver to output the correct solution 74.7\% ($= 0.992^{36.2}$) of the time.
	In 100 epochs, our model learns to correctly solve 63.2\% of boards at test time, reaching 85\% of this theoretical ``best.''
	Hence, our approach demonstrates strong performance in solving visual Sudoku boards end-to-end.
	On the other hand, the baseline convolutional networks make only minuscule improvements to the training loss over the course of 100 epochs, and fail altogether to improve out-of-sample performance.
	Accordingly, our SATNet architecture enables end-to-end learning of the ``rules of the game'' directly from pictorial inputs in a way that was not possible with previous architectures.

	%% Acknowledgements should only appear in the accepted version.
	%\section*{Acknowledgements}
	%
	%\textbf{Do not} include acknowledgements in the initial version of
	%the paper submitted for blind review.
	%
	%If a paper is accepted, the final camera-ready version can (and
	%probably should) include acknowledgements. In this case, please
	%place such acknowledgements in an unnumbered section at the
	%end of the paper. Typically, this will include thanks to reviewers
	%who gave useful comments, to colleagues who contributed to the ideas,
	%and to funding agencies and corporate sponsors that provided financial
	%support.
	
	\section{Conclusion}
	In this paper, we have presented a low-rank differentiable MAXSAT layer that can be integrated into neural network architectures.
	This layer employs block coordinate descent methods to efficiently compute the forward and backward passes, and is amenable to GPU acceleration.
	We show that our SATNet architecture can be successfully used to learn logical structures, namely the parity function and the rules of $9 \times 9$ Sudoku.
	We also show, via a visual Sudoku task, that our layer can be integrated into larger deep network architectures for end-to-end training.
	Our layer thus shows promise in allowing deep networks to learn logical structure \emph{without hard-coding of the relationships between variables}.
	
	More broadly, we believe that this work fills a notable gap in the regime spanning deep learning and logical reasoning.  
	While many ``differentiable logical reasoning'' systems have been proposed, most of them still require 
	fairly hand-specified logical rules and groundings, 
	and thus are somewhat limited in their ability to operate in a truly end-to-end fashion.  
	% Our hope is that by considering a powerful yet generic primitive such as MAXSAT solving and wrapping this within a differentiable framework, 
	Our hope is that by wrapping a powerful yet generic primitive such as MAXSAT solving within a differentiable framework, 
	our solver can enable ``implicit'' logical reasoning to occur where needed within larger frameworks, even if the precise structure of the domain is unknown and must be learned from data.  
	In other words, we believe that SATNet provides a step towards integrating symbolic reasoning and deep learning, a long-standing goal in artificial intelligence. %\footnote{And to finally get Gary Marcus and Yann LeCun to stop arguing, perhaps the greatest achievement of all.} 
	
	\section*{Acknowledgments}
	Po-Wei Wang is supported by a grant from the Bosch Center for AI; Priya Donti is supported by the Department of Energy's Computational Science Graduate Fellowship under grant number DE-FG02-97ER25308; and Bryan Wilder is supported by the National Science Foundation Graduate Research Fellowship.
	
	\bibliography{sdp}
	\bibliographystyle{icml2019}
	
	\clearpage
	\newpage
	\appendix
	\numberwithin{equation}{section}
	\numberwithin{figure}{section}
	\numberwithin{theorem}{section}
	
	\section{Derivation of the forward pass coordinate descent update}
	\label{appsec:coord-desc-fwd}
	
	Our MAXSAT SDP relaxation (described in Section~\ref{sec:sdp}) is given by
	\begin{equation}
		\label{appeq:sdp}
		\begin{aligned}
			\minimize_{V \in \mathbb{R}^{k \times (n+1)}} \;\;&\mdot{S^T S}{V^T V},\\
			\subjectto\;\; &\norm{v_i}=1,\;\;i=0,\ldots,n,
		\end{aligned}
	\end{equation}
	where $S \in \mathbb{R}^{m \times (n+1)}$ and $v_i$ is the $i$th column vector of $V$.
	
	We rewrite the objective of~\eqref{appeq:sdp} as $\mdot{S^T S}{V^T V} \equiv \tr((S^T S)^T (V^T V)) = \tr(V^T V S^T S)$ by noting that $S^T S$ is symmetric and by cycling matrices within the trace. We then observe that the objective terms that depend on any given $v_i$ are given by 
	\begin{equation}
		\label{appeq:coord-desc-obj}
		v_i^T \sum_{j=0}^n s_j^T s_i v_j = v_i^T \sum_{\substack{j=0\\(j \neq i)}}^n s_j^T s_i v_j + v_i^T s_i^T s_i v_i,
	\end{equation}
	where $s_i$ is the $i$th column vector of $S$. Observe $v_i^Tv_i$ in the last term cancels to $1$, and the remaining coefficient
	\begin{equation}
		g_i \equiv \sum_{\substack{j=0\\(j \neq i)}}^n s_j^T s_i v_j = V S^T s_i - \norm{s_i}^2 v_i
	\end{equation} 
	is  constant with respect to $v_i$.
	Thus, \eqref{appeq:coord-desc-obj} can be simply rewritten as 
	\begin{equation}
		v_i^T g_i + s_i^Ts_i.
	\end{equation}
	
	Minimizing this expression over $v_i$ with respect to the constraint  $\|v_i\| = 1$ yields the block coordinate descent update
	\begin{equation}
		\label{appeq:coord-desc}
		v_i = -g_i/\norm{g_i}.
	\end{equation}
	
	\section{Details on backpropagation through the MAXSAT SDP}
	\label{appsec:sdp-backprop}
	
	Given the result $\nicefrac{\partial \ell}{\partial V_{\outset}}$, we next seek to compute $\nicefrac{\partial \ell}{\partial V_{\inset}}$ and $\nicefrac{\partial \ell}{\partial S}$ by pushing gradients through the SDP solution procedure described in Section~\ref{sec:sdp}.
	We do this by taking the total differential through our coordinate descent updates~\eqref{eq:coord-desc-forward} for each output $\outdex \in \outset$ at the optimal fixed-point solution to which these updates converge.
	
	\paragraph{Computing the total differential.} Computing the total differential of the updates~\eqref{eq:coord-desc-forward} and rearranging, we see that for every $\outdex \in \outset$,
	\begin{small}
		\begin{align}
			&\left(\|g_{\outdex}\| I_k -\norm{s_{\outdex}}^2 P_{\outdex} \right)\dd v_{\outdex} + P_{\outdex} \sum_{j\in \outset} s_{\outdex}^T s_j \dd v_j = -P_{\outdex}\xi_{\outdex}, \label{appeq:tot-diff-one}\\
			\intertext{\begin{normalsize}where\end{normalsize}}
			&\xi_{\outdex} \equiv \Big( \sum_{j \in \inset'} s_{\outdex}^Ts_j \dd v_j + V \dd S^Ts_{\outdex} + VS^T \dd s_{\outdex} -2\dd s_{\outdex}^T s_{\outdex} v_{\outdex} \Big), \label{eq:xi}
		\end{align}
	\end{small}
	and where $P_{\outdex} \equiv I_k - v_{\outdex} v_{\outdex}^T, \outdex \in \outset$ and $\inset' \equiv \{ \truthvar \}\, \cup\, \inset$.
	
	\paragraph{Rewriting as a linear system.} Rewriting Equation~\ref{appeq:tot-diff-one} over all $\outdex \in \outset$ as a linear system, we obtain
	\begin{equation}
		\label{appeq:tot-diff}
		\begin{aligned}
			&\Big(\diag(\|g_{\outdex}\|)\otimes I_k + P C \otimes I_k \Big)\vect(\dd V_{\outset}) = -P \vect(\xi_{\outdex}) \\
			&\Rightarrow \vect(\dd V_{\outset}) = -\Big(P (\big(\diag(\|g_{\outdex}\|) + C \big) \otimes I_k )P\Big)^{\dagger} \vect({\xi_{\outdex})},
		\end{aligned}
	\end{equation}
	where $C = S_{\outset}^T S_{\outset} - \diag(\norm{s_{\outdex}}^2)$, $P = \diag(P_{\outdex})$, and the second step follows from the lemma presented in Appendix~\ref{appsec:pseudoinverse-rule}.
	
	We then see that by the chain rule, the gradients $\nicefrac{\partial \ell}{\partial V_{\inset}}$ and $\nicefrac{\partial \ell}{\partial S}$ are given by the left matrix-vector product
	\begin{equation}
		\label{appeq:long}
		\begin{small}
			\begin{aligned}
				&\left( \frac{\partial \ell}{\partial \vect(V_{\outset})}  \right)^T \vect(\dd V_{\outset}) \\
				={ }&
				- \left(\frac{\partial\ell}{\partial \vect(V_{\outset})}\right)^T \Big(P(\big(\diag(\norm{g_{\outdex}})+C\big) \otimes I_k)P\Big)^\dagger \vect(\xi_{\outdex})\\
				%={ }&-\vect\Big(Q\big(\diag(\norm{g_{\outdex}}) + C\big)^\dagger\Big)^TP\vect(\xi_{\outdex}),
			\end{aligned}
		\end{small}
	\end{equation}
	where the second equality comes from plugging in the result of~\eqref{appeq:tot-diff}.
	%, and the last equality comes from the properties of the Kronecker product.
	
	Now, define $U \in \mathbb{R}^{k \times n}$, where the columns $U_{\inset} = 0$ and the columns $U_{\outset}$ are given by
	\begin{equation}\label{appeq:U}
		\vect(U_{\outset})=\Big(P(\big(\diag(\norm{g_{\outdex}})+C\big) \otimes I_k)P\Big)^\dagger\vect\left(\frac{\partial\ell}{\partial \vect(V_{\outset})}\right).
	\end{equation}
% 	Further, define $\hat{U} \in \mathbb{R}^{k \times n}$, where the columns $\hat{U}_{\inset} = 0$ and the columns $\hat{U}_{\outset}$ are given by $\hat{u}_{\outdex} \equiv P_{\outdex} u_{\outdex}$ for $\outdex \in \outset$.
	Then, we see that \eqref{appeq:long} can be written as
	\begin{equation}
		\left( \frac{\partial \ell}{\partial \vect(V_{\outset})}  \right)^T \vect(\dd V_{\outset}) = -\vect(U_{\outset})^T\vect(\xi_{\outdex}),
		\label{eq:app-hat-U}
	\end{equation}
	which is the implicit linear form for our gradients.
	
	\paragraph{Computing desired gradients from implicit linear form.} Once we have obtained $U_{\outset}$ (via coordinate descent), we can explicitly compute the desired gradients $\nicefrac{\partial \ell}{\partial V_{\inset}}$ and $\nicefrac{\partial \ell}{\partial S}$ from the implicit form~\eqref{eq:app-hat-U}.
	For instance, to compute the gradient $\nicefrac{\partial \ell}{\partial v_{\inndex}}$ for some $\inndex \in \inset$, we would set $\dd v_{\inndex} = 1$ and all other gradients to zero in Equation~\eqref{eq:app-hat-U} (where these gradients are captured within the terms $\xi_{\outdex}$).
	
	Explicitly, we compute each $\nicefrac{\partial \ell}{\partial v_{\inndex j}}$ by setting $\dd v_{\inndex j} = 1$ and all other gradients to zero, i.e.
	\begin{equation}
		\begin{aligned}
			\frac{\partial \ell}{\partial v_{\inndex j}} &= -\vect(U_{\outset})^T \vect(\xi_{\outdex}) 
			= -\sum_{\outdex \in \outset} u_{\outdex}^T e_{j} s_{\inndex}^T s_{\outdex} \\
			&= -e_{j}^T \left(\sum_{\outdex \in \outset} u_{\outdex} s_{\outdex}^T   \right) s_{\inndex}.
		\end{aligned}
	\end{equation}
	Similarly, we compute each $\nicefrac{\partial \ell}{\partial S_{i,j}}$ by setting $\dd S_{i,j} = 1$ and all other gradients to zero, i.e.
	\begin{equation}
		\begin{aligned}
			\frac{\partial \ell}{\partial S_{i,j}} &= -\sum_{\outdex \in \outset} u_{\outdex}^T \xi_{\outdex} \\
			&= -\sum_{\outdex \in \outset} u_{\outdex}^T v_i s_{\outdex j} - u_i^T (VS^T)_j + u_i^T (s_{ij} P_i v_i)\\
			&=-v_i^T(\sum_{\outdex \in \outset} u_{\outdex} s_{\outdex j}) - u_i^T (VS^T)_{j}.
		\end{aligned}
	\end{equation}
	In matrix form, these gradients are
	\begin{align}
		&\frac{\partial \ell}{\partial V_{\inset}} = -\left(\sum_{\outdex \in \outset} u_{\outdex} s_{\outdex}^T \right)S_{\inset}, \label{appeq:dl-dvin}\\
		&\frac{\partial \ell}{\partial S} = -\left(\sum_{\outdex \in \outset} u_{\outdex} s_{\outdex}^T \right)^TV - (S V^T)U, \label{appeq:dl-ds}
	\end{align}
	where $u_i$ is the $i$th column of $U$, and where $S_{\inset}$ denotes the $\inset$-indexed column subset of $S$.

	\section{Proof of pseudoinverse computations}
	\label{appsec:pseudoinverse-rule}
	We prove the following lemma, used to derive the implicit total differential for $\vect(\dd V_{\outset})$.
	
	\begin{lemma}\label{lemma:inv} 
	The quantity
		\begin{equation}
		 \vect(\dd V_{\outset}) = \left(P\left( (D+C) \otimes I_k \right) P\right)^{\dagger}\vect(\xi_{\outdex})   
		\end{equation}
	is the solution of the linear system
		\begin{equation}
			(D  \otimes I_k+PC\otimes I_k)\vect(\dd V_{\outset})=P\vect(\xi_{\outdex}),
		\end{equation}
		where $P=\diag(I_k-v_{\outdex} v_{\outdex}^T)$, $C = S_{\outset}^T S_{\outset} - \diag(\norm{s_{\outdex}}^2)$, $D = \diag(\|g_i\|)$, and $\xi_{\outdex}$ is as defined in Equation~\eqref{eq:xi}.
	\end{lemma}
	\begin{proof}
		Examining the equation with respect to $\dd v_i$ gives
		\begin{equation}
			\|g_i\| \dd v_i + P_i \left(\sum_j c_{ij}\dd v_j-\xi_j \right) = 0,
		\end{equation}
		which implies that for all $i$, $\dd v_i=P_iy_i$ for some $y_i$.
		Substituting $y_i$ into the equality gives
		\begin{align}
			&(D \otimes I_k + PC \otimes I_k)P\vect(y_i)\\
			=&P((D +C)\otimes I_k) P\vect(y_i) = P\vect(\xi_{\outdex}).
		\end{align}
		Note that the last equation comes form $D\otimes I_k P = D\otimes I_k PP = P (D\otimes I_k) P $ due to the block diagonal structure of the projection $P$.
		Thus, by the properties of projectors and the pseudoinverse,
		\begin{align}
			\vect(Y) &= (P((D + C)\otimes I_k) P)^\dagger P\vect(\xi_{\outdex})\\
			&=(P((D +C)\otimes I_k) P)^\dagger\vect(\xi_{\outdex}).
		%	&=P((D +C)\otimes I_k)^\dagger P\vect(\xi_{\outdex})\\
		%	&=P ((D +C)^\dagger \otimes I_k) P\vect(\xi_{\outdex}).
		\end{align}
		Note that the first equation comes from the idempotence property of $P$ (that is, $PP=P$).
		Substituting $\vect(\dd V_{\outset}) = P \vect(Y)$ back gives the solution of $\dd V_{\outset}$.
	\end{proof}
	
	\section{Derivation of the backward pass coordinate descent algorithm}
	\label{appsec:back-cd}
	Consider solving for $U_{\outset}$ as mentioned in Equation~\eqref{appeq:U}:
	\begin{equation*}
		\Big(P(\big(\diag(\norm{g_{\outdex}})+C\big) \otimes I_k)P\Big)\vect(U_{\outset})=\vect\left(\frac{\partial\ell}{\partial \vect(V_{\outset})}\right) ,\\
	\end{equation*}
	where $C = S_{\outset}^T S_{\outset} - \diag(\norm{s_{\outdex}}^2)$.
	The linear system can be computed using block coordinate descent.
	Specifically, observe this linear system with respect to only the $u_{\outdex}$ variable. Since we start from $U_{\outset}=0$, we can assume that $P\vect(U_{\outdex})=\vect(U_{\outdex})$. This yields
	\begin{equation}
		\norm{g_{\outdex}} P_\outdex u_{\outdex} +  P_{\outdex}\Big(U_{\outset} S^T_{\outset} s_{\outdex} - \norm{s_{\outdex}}^2 u_{\outdex}\Big) = P_{\outdex}\left(\frac{\partial\ell}{\partial v_{\outdex}}\right).
	\end{equation}
	Let $\Psi = (U_{\outset})S^T_{\outset}$. Then we have
	\begin{equation}
		\norm{g_{\outdex}}P_{\outdex}u_{\outdex} = - P_{\outdex}(\Psi s_{\outdex}-\norm{s_{\outdex}}^2 u_{\outdex} - \partial\ell/\partial v_{\outdex}).
	\end{equation}
	Define $-\dd g_i$ to be the terms contained in parentheses in the right-hand side of the above equation. Note that $\dd g_i$ does not depend on the variable $u_{\outdex}$. Thus, we have the closed-form feasible solution
	\begin{equation}
		u_{\outdex} = -P_{\outdex}\dd g_{\outdex} / \norm{g_{\outdex}}.
	\end{equation}
	After updating $u_{\outdex}$, we can maintain the term $\Psi$ by replacing the old $u^{\text{prev}}_{\outdex}$ with the new $u_{\outdex}$. This yields the rank 1 update
	\begin{equation}
		\Psi := \Psi + (u_{\outdex}-u^{\text{prev}}_{\outdex})s_{\outdex}^T.
	\end{equation}
	The above procedure is summarized in Algorithm~\ref{alg:backward-pass-cd}. Further, we can verify that the assumption $P\vect(U_{\outset})=\vect(U_{\outset})$ still holds after each update by the projection $P_{\outdex}$.
	
	\section{Results for the $4 \times 4$ Sudoku problem}
	\label{appsec:sudoku-4}
	
		% \begin{figure}[h]
	%     \centering
	%     \includegraphics[width=1.5in]{sudoku2x2_loss.pdf}
	%     \includegraphics[width=1.5in]{sudoku2x2_acc.pdf}
	%     \caption{4 by 4 Sudoku: Negative log likelihood (left) and whole-board accuracy (right) of SATNet and Convnet. also need to plot optnet.} 
	%     \label{appfig:sudoku-4}
	% \end{figure}
	
% 	\begin{figure*}[ht!]
% 		\centering
		
% 		\includegraphics[height=1in]{sudoku_4_nll_train_loss.pdf}
% 		\includegraphics[height=1in]{sudoku_4_mse_train_loss.pdf}
% 		\includegraphics[height=1in]{sudoku_4_train_error.pdf}
		
% 		\includegraphics[height=1in]{sudoku_4_nll_test_loss.pdf}
% 		\includegraphics[height=1in]{sudoku_4_mse_test_loss.pdf}
% 		\includegraphics[height=1in]{sudoku_4_test_error.pdf}
		
% 		\caption{$4\times 4$ Sudoku: Loss and whole-board accuracy of each model on the train (top) and test (bottom) sets. \todo{make consistent with figure 5}}
% 		\label{appfig:sudoku-4}
% 	\end{figure*}

    \begin{figure*}[t!]
		\centering
		
		\includegraphics[width=\textwidth]{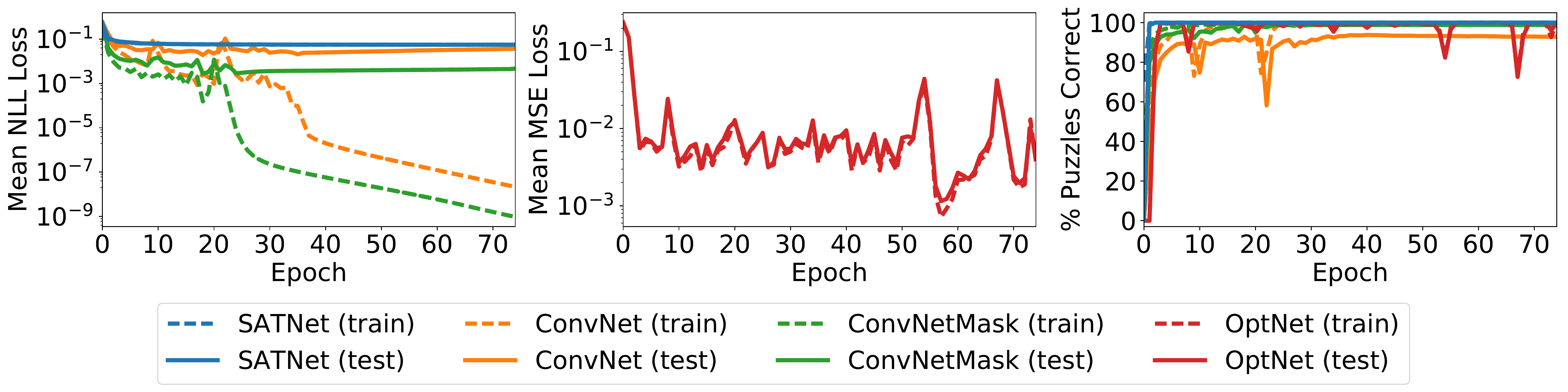}
		
		\caption{Results for $4\times 4$ Sudoku. Lower loss (mean NLL loss and mean MSE loss) and higher whole-board accuracy (\% puzzles correct) are better.}
		\label{appfig:sudoku-4}
	\end{figure*}
	
	We compare the performance of our SATNet architecture on a $4 \times 4$ reduced version of the Sudoku puzzle against OptNet \cite{amos2017optnet} and a convolutional neural network architecture.
	These results (over 9K training and 1K testing examples) are shown in Figure~\ref{appfig:sudoku-4}.
	We note that our architecture converges quickly -- in just two epochs -- to \emph{100\% board-wise test accuracy}. 
	
	OptNet takes slightly longer to converge to similar performance, in terms of both time and epochs.
	In particular, we see that OptNet takes 3-4 epochs to converge (as opposed to 1 epoch for SATNet).
	Further, in our preliminary benchmarks,  OptNet required 12 minutes to run 20 epochs on a GTX 1080 Ti GPU, whereas SATNet took only 2 minutes to run the same number of epochs.
	In other words, we see that SATNet requires fewer epochs to converge \emph{and} takes less time per epoch than OptNet.
	
	Both our SATNet architecture and OptNet outperform the traditional convolutional neural network in this setting, as the ConvNet somewhat overfits to the training set and therefore does not generalize as well to the test set (achieving 93\% accuracy). 
	The ConvNetMask, which additionally receives a binary input mask, performs much better (99\% test accuracy) but does not achieve perfect performance as in the case of OptNet and SATNet.

    \section{Convergence plots for $9 \times 9$ Sudoku experiments}
	\label{appsec:sudoku-9}
    
    Convergence plots for our $9 \times 9$ Sudoku experiments (original and permuted) are shown in Figure~\ref{fig:sudoku9}. 
    SATNet performs nearly identically in both the original and permuted settings, generalizing well to the test set at every epoch without overfitting to the training set.
    The ConvNet and ConvNetMask, on the other hand, do not generalize well.
    In the original setting, both architectures overfit to the training set, showing little-to-no improvement in generalization performance over the course of training.
    In the permuted setting, both ConvNet and ConvNetMask make little progress even on the training set, as they are not able to rely on spatial locality of inputs.
    
    Convergence plots for the visual Sudoku experiments are shown in Figure~\ref{fig:mnist-sudoku9}.
    Here, we see that SATNet generalizes well in terms of loss throughout the training process, and generalizes somewhat well in terms of whole-board accuracy. 
    The difference in generalization performance between the logical and visual Sudoku settings can be attributed to the generalization performance of the MNIST classifier trained end-to-end with our SATNet layer.
    The ConvNetMask architecture overfits to the training set, and the ConvNet architecture makes little-to-no progress even on the training set.  

\begin{figure*}[htb]
	\centering
	\begin{subfigure}{\textwidth}
	\includegraphics[width=\textwidth]{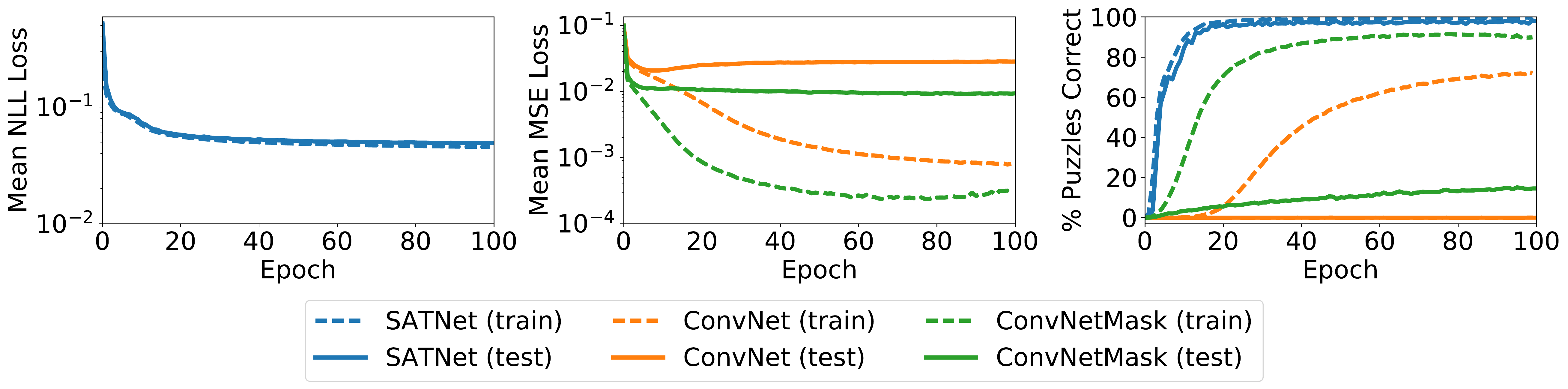}
	\caption{Original $9 \times 9$ Sudoku}
	\end{subfigure}
	\par\bigskip
% 	\par\bigskip

    \begin{subfigure}{\textwidth}
	\includegraphics[width=\textwidth]{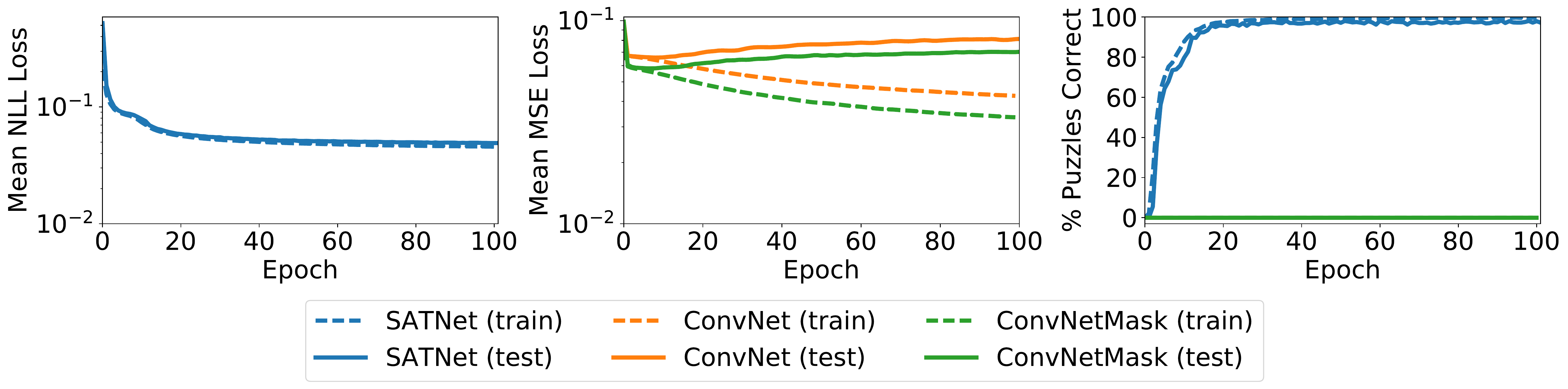}
	\caption{Permuted $9 \times 9$ Sudoku}
	\end{subfigure}
% 	\par\bigskip
		\caption{Results for our $9 \times 9$ Sudoku experiments. 
		Lower loss (mean NLL loss and mean MSE loss) and higher whole-board accuracy (\% puzzles correct) are better.}	\label{fig:sudoku9}
	\end{figure*}
\par\bigskip
\par\bigskip
\begin{figure*}[htb]
	\centering
	\includegraphics[width=\textwidth]{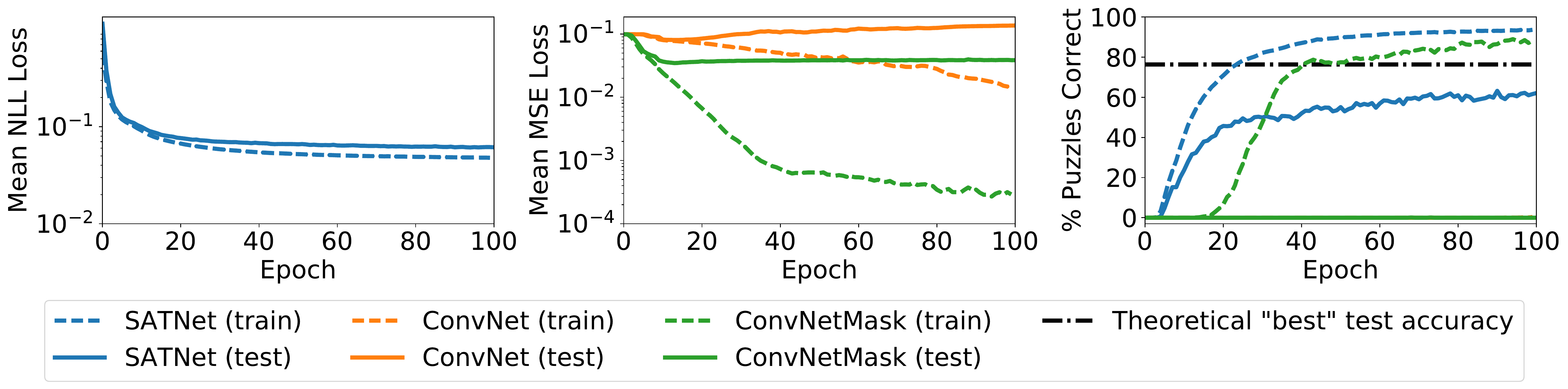}
	\caption{Results for our visual Sudoku experiments. 
		Lower loss (mean NLL loss and mean MSE loss) and higher whole-board accuracy (\% puzzles correct) are better. The theoretical ``best'' test accuracy plotted is for our specific choice of MNIST classifier architecture.}	\label{fig:mnist-sudoku9}
	\end{figure*}

\end{document}